\documentclass[runningheads]{llncs}

\usepackage[
sortcites,
style=lncs,
backend=biber,
doi=false,
url=false]{biblatex}
\usepackage[T1]{fontenc}
\usepackage[utf8]{inputenc}
\usepackage{graphicx}
\usepackage{amsmath,amsfonts,amssymb,amsfonts}
\usepackage[algo2e,ruled,vlined,linesnumbered]{algorithm2e}
\usepackage{hhline}
\usepackage{tabu}
\usepackage{tikz}
\usepackage{array}
\usepackage{multirow}
\usepackage{tabularx}
\usepackage{booktabs}
\usepackage{calligra}
\usepackage{mathtools}
\usepackage{subcaption}
\usepackage{enumitem}
\usepackage{bbding}

\usepackage{mdframed}
\usepackage{hyperref}
\usepackage{multirow}

\newtheorem{assumption}{}

\renewcommand{\qed}{\hfill \ensuremath{\Box}}

\newcommand{\Id}{\mathbf{I}}

\newcommand{\R}{\mathbb{R}}
\newcommand{\pa}[1]{\left( #1 \right)}
\newcommand{\abs}[1]{\left| #1 \right|}
\newcommand{\norm}[1]{\left\lVert #1 \right\rVert}
\newcommand{\norminf}[1]{\left\lVert #1 \right\rVert_{\infty}}
\newcommand{\normf}[1]{\left\lVert #1 \right\rVert_{F}}

\DeclareMathAlphabet{\mathcalligra}{T1}{calligra}{m}{n}

\newcommand{\fop}{\Av}

\newcommand{\E}{\mathbb{E}}

\newcommand{\loss}{\mathcal{L}}

\newcommand{\tp}{^\top}
\newcommand{\sqb}[1]{\left[ #1 \right]}
\newcommand{\prob}[1]{\mathbb{P}\left( #1 \right)}

\newcommand{\ran}[1]{\mathrm{Im}\left(#1\right)}

\newcommand{\Expect}[2]{\E_{#1}\sqb{#2}}

\newcommand{\Av}{\mathbf{A}}
\newcommand{\Mv}{\mathbf{M}}
\newcommand{\zv}{\mathbf{z}}
\newcommand{\yv}{\mathbf{y}}

\newcommand{\xv}{\mathbf{x}}
\newcommand{\xvz}{\xv_0}
\newcommand{\xvc}{\overline{\xv}}
\newcommand{\uv}{\mathbf{u}}

\newcommand{\thetav}{\pmb{\theta}}

\newcommand{\thetavt}{\thetav(t)}
\newcommand{\thetavs}{\thetav(s)}

\newcommand{\thetavz}{\thetav_0}

\newcommand{\gdipz}{\mathbf{g}(\uv,\Wv(0))}
\newcommand{\gdip}{\mathbf{g}(\uv,\thetav)}
\newcommand{\gdipt}{\mathbf{g}(\uv, \thetavt)}

\newcommand{\jcal}{\mathcal{J}}
\newcommand{\jtheta}{\mathcal{J}(\thetav)}
\newcommand{\jthetat}{\mathcal{J}(\thetav\left(t\right))}
\newcommand{\jthetas}{\mathcal{J}(\thetav\left(s\right))}

\newcommand{\jthetaz}{\mathcal{J}(\thetavz)}
\newcommand{\Hv}{\mathbf{H}}

\newcommand{\Wv}{\mathbf{W}}

\newcommand{\Wvalt}{\mathbf{\widetilde{W}}}
\newcommand{\Vv}{\mathbf{V}}

\newcommand{\jW}{\mathcal{J}(\Wv)}

\newcommand{\jWalt}{\mathcal{J}(\Wvalt)}

\newcommand{\yvt}{\yv(t)}
\newcommand{\yvz}{\yv(0)}
\newcommand{\yvc}{\overline{\yv}}

\newcommand{\ysy}{\yv(s) - \yv}%
\newcommand{\yty}[1][]{%
\ifthenelse{\equal{#1}{}}{\yvt - \yv}{\yv(#1) - \yv}%
}
\newcommand{\yzy}[1][]{%
\ifthenelse{\equal{#1}{}}{\yvz - \yv}{\yv(#1) - \yv}%
}
\newcommand{\ytypar}{(\yty)}

\newcommand{\Cphi}{C_{\phi}}
\newcommand{\Cphid}{C_{\phi'}}

\newcommand{\stddistrib}{\mathcal{N}(0,1)}

\newcommand{\Ball}{\mathbb{B}}
\newcommand{\sph}{\mathbb{S}}

\newcommand{\Lip}{\mathrm{Lip}}

\newcommand{\lammin}{\lambda_{\min}}
\newcommand{\lammax}{\lambda_{\max}}
\newcommand{\sigmin}{\sigma_{\min}}
\newcommand{\sigmax}{\sigma_{\max}}
\newcommand{\sigminA}{\sigma_{\fop}}

\newcommand{\deriv}[2]{\frac{\mathrm{d} #1}{\mathrm{d} #2}}

\begin{document}

\title{Convergence Guarantees of Overparametrized Wide Deep Inverse Prior}
\titlerunning{Overparametrization of DIP Networks}

\author{Nathan Buskulic\Envelope \and Yvain Qu\'eau \and Jalal Fadili}
\authorrunning{N. Buskulic et al.}

\institute{Normandie Univ., UNICAEN, ENSICAEN, CNRS, GREYC, Caen, France\\ 
\Envelope \href{mailto:nathan.buskulic@unicaen.fr}{\texttt{nathan.buskulic@unicaen.fr}}}


\maketitle


\begin{abstract}
Neural networks have become a prominent approach to solve inverse problems in recent years. Amongst the different existing methods, the Deep Image/Inverse Priors (DIPs) technique is an unsupervised approach that optimizes a highly overparametrized neural network to transform a random input into an object whose image under the forward model matches the observation. However, the level of overparametrization necessary for such methods remains an open problem. In this work, we aim to investigate this question for a two-layers neural network with a smooth activation function. We provide overparametrization bounds under which such network trained via continuous-time gradient descent will converge exponentially fast with high probability which allows to derive recovery prediction bounds. This work is thus a first step towards a theoretical understanding of overparametrized DIP networks, and more broadly it participates to the theoretical understanding of neural networks in inverse problem settings. 





\end{abstract}

 \vspace*{-1em}



\section{Introduction}
\subsection{Problem Statement}
A linear inverse problem consists in reliably recovering an object $\xvc \in \R^n$ from noisy indirect observations
\begin{align}\label{eq:forward}
\yv = \fop \xvc + \epsilon ,
\end{align}
where $\yv \in \R^m$ is the observation, $\fop\in\R^{m\times n}$ is a linear forward operator, and $\epsilon$ stands for some additive noise. Without loss of generality, we will assume throughout that $\yv \in \ran{\fop}$.

In recent years, the use of sophisticated machine learning algorithms, including deep learning, to solve inverse problems has gained a lot of momentum and provides promising results, see e.g., reviews \cite{arridge_solving_2019,ongie_deep_2020}. Most of these methods are supervised and require extensive datasets for training, which might not be available. An interesting unsupervised alternative~\cite{ulyanov_deep_2020} is known as Deep Image Prior, which is also named Deep Inverse Prior (DIP) as it is not confined to images. 
In the DIP framework, a generator network $\mathbf{g}: (\uv,\thetav) \in \R^d\times \R^p \mapsto \xv \in \R^n$, with activation function $\phi$, is optimized to transform some random input $\uv \in \R^d$ into a vector in $\xv \in \R^n$. The parameters $\thetav$ of the network are optimized via (possibly stochastic) gradient descent to minimize the squared Euclidean loss
\begin{align}\label{eq:min}
\loss(\gdip) = \frac{1}{2m}\norm{\fop\gdip - \yv}^2.
\end{align}
Theoretical understanding of recovery and establishing convergence guarantees for deep learning-based methods is of paramount importance to make their routine usage in critical applications reliable \cite{mukherjee_learned_2022}. Our goal in this paper is to participate to this endeavour by explaining when gradient descent consistently and provably finds global minima of \eqref{eq:min}, and how this translates into recovery guarantees of \eqref{eq:forward}. For this, we focus on a continuous-time gradient flow applied to \eqref{eq:min}:
\begin{align}
\begin{cases}
\dot{\thetav}(t) = - \nabla_{\thetav} \loss(\gdipt),  \\
\thetav(0) = \thetav_0 .
\end{cases}\label{eq:gradflow}
\end{align}
This is an idealistic setting which makes the presentation simpler and it is expected to reflect the behavior of practical and commonly encountered first-order descent algorithms, as they are known to approximate gradient flows.

\subsection{Contributions} 
We will deliver a first theoretical analysis of DIP models in the overparametrized regime. We will first analyze \eqref{eq:gradflow} by providing sufficient conditions for $\yvt := \fop \gdipt$ to converge exponentially fast to a globally optimal solution in the observation space. This result is then converted to a prediction error on $\yvc := \fop\xvc$ through an early stopping strategy. Our conditions and bounds involve the conditioning of the forward operator, the minimum and maximum singular values of the Jacobian of the network, as well as its Lipschitz constant. We will then turn to evaluating these quantities for the case of a two-layer neural network 
\begin{align}\label{eq:dipntk}
    \gdip = \frac{1}{\sqrt{k}}\Vv \phi(\Wv\uv),
\end{align}
with  $\Vv \in \R^{n \times k}$ and $\Wv \times \R^{k \times d}$, and $\phi$ an element-wise nonlinear activation function. The scaling by $\sqrt{k}$ will become clearer later. In this context, the network will be optimized with respect to the first layer (i.e., $\Wv$) while keeping the second (i.e., $\Vv$) fixed. Consequently, $\thetav=\Wv$. We show that for a proper random initialization $\Wv(0)$ and sufficient overparametrization, all our conditions are in force and the smallest eigenvalue of the Jacobian is indeed bounded away from zero independently of time. We provide a characterization of the overparametrization needed in terms of $(k,d,n)$ and the conditioning of $\fop$. Lastly, we show empirically that the behavior of real-world DIP networks is consistent with our theoretical bounds.

\subsection{Relation to Prior Work}\label{sec:prior}
\subsubsection{Data-Driven Methods to Solve Inverse Problems}
Data-driven approaches to solve inverse problems come in various forms~\cite{arridge_solving_2019,ongie_deep_2020}. The first type trains an end-to-end network to directly map the observations to the signals for a specific problem. 
While they can provide impressive results, these methods can prove very unstable~\cite{mukherjee_learned_2022} as they do not use the physics of the problem which can be severely ill-posed. To cope with these problems, several hybrid models that mix model- and data-driven algorithms were developed in various ways. One can learn the regularizer of a variational problem  \cite{prost_learning_2021} or use Plug-and-Play methods \cite{venkatakrishnan_plug-and-play_2013} for example. Another family of approaches, which takes inspiration from classical iterative optimization algorithms, is based on unrolling 
(see \cite{monga_algorithm_2021} for a review of these methods). 
Still, all these methods require extensive amount of training data, which may not always be available. Their theoretical recovery guarantees are also not well understood \cite{mukherjee_learned_2022}.

\vspace*{-.5em}

\subsubsection{Deep Inverse Prior}
The DIP model \cite{ulyanov_deep_2020} (and its extensions that mitigate some of its empirical issues ~\cite{liu_image_2019,mataev_deepred_2019,shi_measuring_2022,zukerman_bp-dip_2021}) is an unsupervised alternative to the supervised approches briefly reviewed above. The empirical idea is that the architecture of the network acts as an implicit regularizer and will learn a more meaningful transformation before overfitting to artefacts or noise. With an early stopping strategy, one can get the network to generate a vector close to the sought signal. However, this remains purely empirical and there is no guarantee that a network trained in such manner converges in the observation space (and even less in the signal space). Our work aims at reducing this theoretical gap, by analyzing the behaviour of the network in the observation (prediction) space.

\vspace*{-.5em}

\subsubsection{Theory of Overparametrized Networks}
In parallel to empirical studies, there has been a lot of effort to develop some theoretical understanding of the optimization of overparametrized networks \cite{bartlett_deep_2021,fang_mathematical_2021}. Amongst the theoretical models that emerged to analyze neural networks, the Neural Tangent Kernel (NTK) captures the behavior of neural networks in the infinite width limit during optimization via gradient descent. In the NTK framework, the neural network behaves as its linearization around the initialization, thus yielding a model equivalent to learning with a specific positive-definite kernel (so-called NTK). In \cite{jacot_neural_2018}, it was shown that in a highly overparametrized regime and random initialization, parameters $\thetavt$ stay near the initialization, and are well approximated by their linearized counterparts at all times (also called the ``lazy'' regime in \cite{chizat_lazy_2019}). With a similar aim, several works characterized the overparametrization necessary to obtain similar behaviour for shallow networks, see e.g., \cite{du_gradient_2019,oymak_overparameterized_2019,allen-zhu_convergence_2019,oymak_toward_2020}. All these works provide lower bounds on the number of neurons from which they can prove convergence rates to a zero-loss solution. Despite some apparent similarities, our setting has important differences. On the one hand, we have indirect measurements through (fixed) $\fop$, the output is not scalar, and there is no supervision. On the other hand, unlike all above works which deal with a supervised training setting, in the DIP model the dimension $d$ of the input is a free parameter, while it is imposed in a supervised setting.
\section{DIP Guarantees}\label{sec:main_res}
\subsection{Notations}
For a matrix $\Mv \in \R^{a \times b}$ we denote, when dimension requirements are met, by $\lammin(\Mv)$ and $\lammax(\Mv)$ (resp. $\sigmin(\Mv)$ and $\sigmax(\Mv)$) its smallest and largest eigenvalues (resp. non-zero singular values), and by $\kappa(\Mv) = \frac{\sigmax(\Mv)}{\sigmin(\Mv)}$ its condition number. We also denote by $\normf{\cdot}$ the Frobenius norm and $\norm{\cdot}$ the Euclidean norm of a vector (or operator norm of a matrix). We use $\Mv^i$ (resp. $\Mv_i$) as the $i$-th row (resp. column) of $\Mv$. We represent a ball of radius $r$ and center $x$ by $\Ball(x,r)$. 
We also define $\yvt = \fop \gdipt$ and $\yvc = \fop\xvc$. The Jacobian of the network is denoted $\jthetat$. The Lipschitz constant of a mapping is denoted $\Lip(\cdot)$. We set $\Cphi = \sqrt{\Expect{g\sim\stddistrib}{\phi(g)^2}}$ and $\Cphid = \sqrt{\Expect{g\sim\stddistrib}{\phi'(g)^2}}$ with $\Expect{}{X}$ the expected value of $X$. 

\subsection{Main Result}

\paragraph{\textbf{Standing Assumptions}} 
In the rest of this work, we assume that:
\begin{mdframed}
\begin{assumption}\label{ass:u_sphere}
$\uv$ is drawn uniformly on $\sph^{d-1}$; \vspace*{-.75em}
\end{assumption}
\begin{assumption}\label{ass:w_init}
$\Wv(0)$ has iid entries from $\stddistrib$; \vspace*{-.75em}
\end{assumption}
\begin{assumption}\label{ass:v_init}
$\Vv$ has iid columns with identity covariance and $D$-bounded entries;  \vspace*{-.75em}
\end{assumption}
\begin{assumption}\label{ass:phi_diff}
$\phi$ is a twice differentiable function with $B$-bounded derivatives.  
\end{assumption}
\end{mdframed}



Assumptions~\ref{ass:u_sphere}, \ref{ass:w_init} and \ref{ass:v_init} are standard. Assumptions~\ref{ass:phi_diff} is met by many activations such as the softmax, sigmoid or hyperbolic tangent. Including the ReLU would require more technicalities that will be avoided here. 

\paragraph{\textbf{Well-posedness}} In order for our analysis to hold, the Cauchy problem~\eqref{eq:gradflow} needs to be well-defined. This is easy to prove upon observing that under \eqref{ass:phi_diff}, the gradient of the loss is both Lipschitz and continuous. Thus, the Cauchy-Lipschitz theorem applies, ensuring that \eqref{eq:gradflow} has a unique global continuously differentiable solution trajectory.


Our main result establishes the prediction error for the DIP model.
\begin{theorem}\label{thm:main}
Consider a network $\gdip$, with $\phi$ obeying \eqref{ass:phi_diff}, optimized via the gradient flow \eqref{eq:gradflow}.
\begin{enumerate}[label=(\roman*)]
\item Let $\sigminA = \inf_{\zv \in \ran{\fop}}\norm{\fop\tp \zv}/\norm{\zv} > 0$. Suppose that
\begin{equation}\label{eq:bndR}
\frac{\norm{\yv - \fop \mathbf{g}(\uv,\thetavz)}}{\sigminA} < \frac{\sigmin(\jthetaz)^2}{4\Lip(\jcal)} .
\end{equation}
Then for any $\epsilon > 0$
\begin{align}\label{eq:convexp}
\norm{\yvt - \yvc} \leq 2 \norm{\epsilon} \qquad \text{for all} \qquad t \geq \frac{4m\log\left(\norm{\yv - \fop \mathbf{g}(\uv,\thetavz)}/\norm{\epsilon}\right)}{\sigminA^2\sigmin(\jthetaz)^2} .
\end{align}

\item Let the one-hidden layer network \eqref{eq:dipntk} with architecture parameters obeying
\begin{align*}
k \geq C_1 \kappa(\fop)^2n \pa{\sqrt{n}\pa{\sqrt{\log(d)} + 1} + \sqrt{m}}^2    
\end{align*}
Then 
\begin{align*}
\norm{\yvt - \yvc} \leq 2 \norm{\epsilon} \qquad \text{for all} \qquad t \geq \frac{C_2m\log\left(\norm{\yv - \fop \mathbf{g}(\uv,\Wv(0))}\right)}{\sigminA^2\Cphid^2} 
\end{align*}
with probability at least $1 - n^{-1} - d^{-1}$, where $C_i$ are positive constants that depend only on the activation function and the bound $D$.
\end{enumerate}
\end{theorem}

%
%

Before proceeding with the proof, a few remarks are in order: 
\begin{enumerate}[label=$\bullet$]
\item We start with the scaling of the network architecture parameters required. First, the bound on $k$, the number of neurons of the hidden layer, scales quadratically in $n^2$ and linearly in $m$. We thus have the bound $k \gtrsim n^2m$. The probability of success in our theorem is also dependent on the architecture parameters. More precisely, this probability increases with growing number of observations.\\[.25em]

\item The other scaling of the theorem is on the input size $d$ and informs us that its influence is logarithmic. The bound is more demanding as $\fop$ becomes more ill-conditioned. The latter dependency can be interpreted as follows: the more ill-conditioned the original problem is, the larger the network needs to be. Let us emphasize that, contrary to other learning settings in the literature, the size $d$ of the random input  $\uv$ is free, and so far it has remained unclear how to choose it. Our result provides a first  answer for shallow networks in the overparametrized setting: most of the overparametrization necessary for the optimization to converge is due to $k$.\\[.25em]

\item On our way to prove \eqref{thm:main}, we actually show that $\yv(t)$ converges exponentially to $\yv$, which is converted to a recovery of $\yvc$ through an early stopping strategy. This ensures that the network does not overfit the noise and provides a solution in a ball around $\yvc$ whose radius is linear in the noise level (so-called prediction linear convergence in the inverse problem literature). This provides a first result on convergence of wide DIP networks that ensures they behave well in the observation space. \\[.25em]

\item One has to keep in mind, however, that Theorem~\ref{thm:main} does not say anything about the recovered vector generated by the network and its relation to $\xvc$ (in absence of noise and at convergence, it might be any element of $\xvc + \text{ker}(\fop)$). Of course, when $\fop$ is invertible, then we are done. In the general case, this is a much more challenging question which requires a more involved analysis and a restricted-type injectivity assumption. This will be the subject of a forthcoming paper.
\end{enumerate}


\section{Proof}
The proof consists of two main steps. First, we prove that under \eqref{eq:bndR}, $\yv(t)$ converges exponentially fast with a time-independent rate. We then use a triangle inequality and an early stopping criterion to show our result. The proof of the second claim will consist in verifying that \eqref{eq:bndR} holds with high probability for our random model of the two-layer network under our scaling. Both proofs rely on several technical lemmas, which will be given later.

\begin{enumerate}[label=(\roman*)]
\item The solution trajectory $\thetav(t)$, hence $\yv(t)$, is continuously differentiable, and thus
\begin{align}\label{eq:maindiffineq}
&\deriv{\frac{1}{2}\norm{\yty}^2}{t} 
= \ytypar\dot{\yv}(t)  \nonumber\\
&= \ytypar \fop \jthetat \dot{\thetav}(t)\nonumber\\
&= - \ytypar \fop \jthetat \nabla_{\thetav} \loss(\gdipt)\nonumber\\
&= - \frac{1}{m}\ytypar\tp \fop \jthetat \jthetat\tp \fop\tp \ytypar \nonumber\\
&= - \frac{1}{m}\norm{\jthetat\tp \fop\tp \ytypar}^2 \leq - \frac{\sigmin(\jthetat)^2\sigminA^2}{m} \norm{\yty}^2,
\end{align}
where we used the fact that $\yty \in \ran{\fop}$. In view of Lemma~\ref{lemma:link_params_singvals}\ref{claim:sigval_bounded_everywhere}, we have $\sigmin(\jthetat) \geq \sigmin(\jthetaz)/2$ for all $t \geq 0$ if the initialization error verifies~\eqref{eq:bndR}, and in turn 
\begin{align*}
\deriv{\norm{\yty}^2}{t} 
&\leq - \frac{\sigmin(\jthetaz)^2\sigminA^2}{2m} \norm{\yty}^2 .
\end{align*}
Integrating, we obtain 
\begin{align}\label{eq:expoconvy}
\norm{\yty} \leq \norm{\yvz-\yv} e^{-\tfrac{\sigmin(\jthetaz)^2\sigminA^2}{4m}t} .
\end{align}
Using 
\[
\norm{\yty} \leq \norm{\yvt - \yvc} + \norm{\epsilon} \leq \norm{\yvz-\yv} e^{-\tfrac{\sigmin(\jthetaz)^2\sigminA^2}{4m}t} + \norm{\epsilon}, 
\]
we get the early stopping bound by bounding the exponential term by $\norm{\epsilon}$.

\item To show the statement, it is sufficient to check that \eqref{eq:bndR} holds under our scaling. From Lemma~\ref{lemma:min_eigenvalue_singvalue_init}, we have
\begin{align*}
\sigmin(\jthetaz) \geq \Cphid/2 
\end{align*}
with probability at least $1-n^{-1}$ provided $k \geq C_0 n\log(n)$ for $C_0 > 0$. Combining this with Lemma~\ref{lemma:lip-jacobian} and Lemma~\ref{lemma:bound_initial_misfit} and the union bound, it is sufficient for \eqref{eq:bndR} to be fulfilled with probability at least $1-n^{-1}-d^{-1}$, that
\[
C_1\kappa(\fop)\pa{\sqrt{n}\pa{\sqrt{\log(d)}+1} + \sqrt{m}} < \frac{\Cphid^2\sqrt{k}}{16BD \sqrt{n}} .
\]
\end{enumerate}

We now prove the intermediate lemmas invoked in the proof. 


\begin{lemma}\label{lemma:link_params_singvals}
    \begin{enumerate}[label=(\roman*)]
        \item \label{claim:singvals_bounded_if_params_bounded} If $\thetav \in \Ball(\thetavz,R)$ with $R=\frac{\sigmin(\jthetaz)}{2\Lip(\jcal)}$, then
        \begin{align*}
            \sigmin(\jtheta) \geq \sigmin(\jthetaz)/2.
        \end{align*}
        
        \item  \label{claim:params_bounded_if_singvals_bounded} If for all $s \in [0,t] $, $\sigmin(\jthetas) \geq \frac{\sigmin(\jthetaz}{2}$, then 
            \begin{align*}
                \thetavt \in \Ball(\thetavz,R') \qquad \text{with} \qquad R' = \frac{2}{\sigminA\sigmin(\jthetaz)}\norm{\yzy}.
            \end{align*}
        
        \item \label{claim:sigval_bounded_everywhere}
        If $R'<R$, then for all $t \geq 0$, $\sigmin(\jthetat) \geq \sigmin(\jthetaz)/2$.
\end{enumerate}


\end{lemma}
\begin{proof}
\begin{enumerate}[label=(\roman*)]
\item Since $\thetav \in \Ball(\thetavz,R)$, we have
    \begin{align*}
        \norm{\jtheta - \jthetaz} \leq \Lip(\jcal)\norm{\thetav - \thetavz} \leq \Lip(\jcal)R.
    \end{align*}
    By using that $\sigmin(A)$ is 1-Lipschitz, we obtain
    \begin{align*}
        \sigmin(\jtheta) \geq \sigmin(\jthetaz) - \norm{\jtheta - \jthetaz} \geq \frac{\sigmin(\jthetaz}{2} .
    \end{align*}

\item From \eqref{eq:maindiffineq}, we have for all $s \in [0,t]$
\begin{align*}
    \deriv{\norm{\ysy}}{s} &= - \frac{1}{m}\frac{\norm{\jthetas\tp\fop\tp(\ysy)}^2}{\norm{\ysy}}\\
    &\leq - \frac{\sigmin(\jthetaz)\sigminA}{2m}\norm{\jthetas\tp\fop\tp(\ysy)}.
\end{align*}
The Cauchy-Schwarz inequality and \eqref{eq:gradflow} imply that
\begin{align*}
\deriv{\norm{\thetavs - \thetavz}}{s} = \frac{\dot{\thetav}(s)\tp\pa{\thetavs - \thetavz}}{\norm{\thetavs - \thetavz}} \leq \norm{\dot{\thetav}(s)} = \frac{1}{m}\norm{\jthetas\tp\fop\tp(\ysy)}.
\end{align*}
We therefore get
\begin{align*}
\deriv{\norm{\thetavs - \thetavz}}{s} + \frac{2}{\sigmin(\jthetaz)\sigminA}\deriv{\norm{\ysy}}{s} \leq 0 .
\end{align*}
Integrating over $s \in [0,t]$, we get the claim.

\item Actually, we prove the stronger statement that $\thetavt \in \Ball(\thetavz,R')$ for all $t \geq 0$, whence our claim will follow thanks to \ref{claim:singvals_bounded_if_params_bounded}. Let us assume for contradiction that $R'<R$ and $\exists~ t < +\infty$ such that $\thetavt \notin \Ball(\thetavz,R')$. By \ref{claim:params_bounded_if_singvals_bounded}, this means that $\exists~ s \leq t$ such that $\sigmin(\jthetas) < \sigmin(\jthetaz)/2$. In turn, \ref{claim:singvals_bounded_if_params_bounded} implies that $\thetav(s) \notin \Ball(\thetavz,R)$. Let us define
    \begin{align*}
        t_0 = \inf\{\tau \geq 0:\thetav(\tau) \notin \Ball(\thetavz,R)\},
    \end{align*}
    which is well-defined as it is at most $s$. Thus, for any small $\epsilon > 0$ and for all $t' \leq t_0 - \epsilon$, $\thetav(t') \in \Ball(\thetavz,R)$ which, in view of \ref{claim:singvals_bounded_if_params_bounded} entails that $\sigmin(\jtheta(t')) \geq \sigmin(\jthetaz)/2$. In turn, we get from \ref{claim:params_bounded_if_singvals_bounded} that $\thetav(t_0-\epsilon) \in \Ball(\thetavz,R')$. Since $\epsilon$ is arbitrary and $\thetav$ is continuous, we pass to the limit as $\epsilon \to 0$ to deduce that $\thetav(t_0) \in \Ball(\thetavz,R') \subsetneq \Ball(\thetavz,R)$ hence contradicting the definition of $t_0$.
\qed
\end{enumerate}
\end{proof}

\begin{lemma}[Bound on $\sigmin(\jthetaz)$]
\label{lemma:min_eigenvalue_singvalue_init}
For the one-hidden layer network \eqref{eq:dipntk}, under assumptions \eqref{ass:u_sphere}-\eqref{ass:phi_diff}. We have
\begin{align*}
\sigmin(\jthetaz) \geq \Cphid/2 
\end{align*}
with probability at least $1-n^{-1}$ provided $k \geq C n\log(n)$ for $C > 0$ large enough that depends only on $\phi$ and the bound on the  entries of $\Vv$.
\end{lemma}

\begin{proof}
Define the matrix $\Hv = \jthetaz\jthetaz\tp$. For the two-layer network, and since $\uv$ is on the unit sphere, $\Hv$ reads
\[
\Hv = \frac{1}{k} \sum_{i=1}^k \phi'(\Wv^i(0)\uv)^2\Vv_i\Vv_i\tp .
\]
It follows that
\[
\Expect{}{\Hv} = \Expect{g\sim \stddistrib}{\phi'(g)^2}\frac{1}{k} \sum_{i=1}^k\Expect{}{\Vv_i \Vv_i\tp} = \Cphid^2 \Id_n ,
\]
where we used \ref{ass:u_sphere}-\ref{ass:w_init} and orthogonal invariance of the Gaussian distribution, hence $\Wv^i(0)\uv$ are iid $\stddistrib$, as well as \ref{ass:v_init} and independence between $\Vv$ and $\Wv(0)$. Moreover,
\[
\lammax(\phi'(\Wv^i(0)\uv)^2\Vv_i\Vv_i\tp) \leq B^2 D^2 n .
\]
We can then apply the matrix Chernoff inequality \cite[Theorem~5.1.1]{tropp_introduction_2015} to get
\[
\prob{\sigmin(\jthetaz) \leq \delta\Cphid} \leq ne^{-\frac{(1-\delta)^2k\Cphid^2}{2B^2 D^2 n}} .
\]
Taking $\delta=1/2$ and $k$ as prescribed, we conclude.
\qed
\end{proof}

\begin{lemma}[Lipschitz constant of the Jacobian]\label{lemma:lip-jacobian}
For the one-hidden layer network \eqref{eq:dipntk}, under assumptions \eqref{ass:u_sphere}, \eqref{ass:w_init} and \eqref{ass:phi_diff}, we have
\[
\Lip(\jcal) \leq B D \sqrt{\frac{n}{k}} .
\]
\end{lemma}

\begin{proof}
We have for all $\Wv, \Wvalt \in \R^{k \times d}$, 
\begin{align*}
\norm{\jW - \jWalt}^2 
&\leq \frac{1}{k} \sum_{i=1}^k |\phi'(\Wv^i\uv) - \phi'(\Wvalt^i\uv)|^2 \normf{\Vv_i\uv\tp}^2 \\
&=\frac{1}{k} \sum_{i=1}^k |\phi'(\Wv^i\uv) - \phi'(\Wvalt^i\uv)|^2 \norm{\Vv_i}^2 \\
&\leq B^2D^2\frac{n}{k} \sum_{i=1}^k |\Wv^i\uv - \Wvalt^i\uv|^2 \\
&\leq B^2D^2\frac{n}{k} \sum_{i=1}^k \norm{\Wv^i - \Wvalt^i}^2 = B^2D^2\frac{n}{k} \normf{\Wv - \Wvalt}^2 .
\end{align*}
\qed
\end{proof}

\begin{lemma}[Bound on the initial error]\label{lemma:bound_initial_misfit}
Under the main assumptions, the initial error of the network is bounded by
\begin{align*}
\norm{\yv(0) - \yv} \leq \norm{\fop}\pa{C \sqrt{n\log(d)} + \sqrt{n}\norminf{\xvz} + \sqrt{m}\norminf{\epsilon}},
\end{align*}
with probability at least $1 - d^{-1}$.
\end{lemma}

\begin{proof}
We first observe that
\[
\norm{\yvz - \yv} \leq \norm{\fop}\norm{\gdipz} + \norm{\fop}\pa{\sqrt{n} \norminf{\xvz} + \sqrt{m} \norminf{\epsilon}} ,
\]
where $\gdipz = \frac{1}{\sqrt{k}} \sum_{i=1}^k \phi(\Wv^i\uv)\Vv_i$. We now prove that this term concentrates around its expectation. First, we have by independence
\[
\Expect{}{\norm{\gdipz}}^2 \leq \frac{1}{k}\Expect{}{\norm{\sum_{i=1}^k \phi(\Wv^i\uv)\Vv_i}^2} = \Expect{}{\phi(\Wv^1\uv)^2\norm{\Vv_1}^2} = \Cphi^2 n .
\]
In addition,
\begin{align*}
&\abs{\norm{\mathbf{g}(\uv,\Wv)} - \lVert\mathbf{g}(\uv,\Wvalt)\rVert} 
\leq \frac{1}{\sqrt{k}}\norm{\sum_{i=1}^k \pa{\phi(\Wv^i\uv) - \phi(\Wvalt^i\uv)}\Vv_i} \\
&\leq BD\sqrt{n}\pa{\frac{1}{\sqrt{k}}\sum_{i=1}^k \norm{\Wv^i - \Wvalt^i}} \leq BD\sqrt{n}\normf{\Wv - \Wvalt} .
\end{align*}
We then get
\begin{align*}
&\prob{\norm{\mathbf{g}(\uv,\Wv(0))} \geq \Cphi \sqrt{n\log(d)} + \tau} \\
&\leq \prob{\norm{\mathbf{g}(\uv,\Wv(0))} \geq \Expect{}{\norm{\gdipz}} + \tau}\\ 
&\leq e^{-\frac{\tau^2}{2\Lip(\norm{\mathbf{g}(\uv,\Wv(0))})^2}} \leq e^{-\frac{\tau^2}{2nB^2D^2}} .
\end{align*}
Taking $\tau = \sqrt{2}BD \sqrt{n\log(d)}$, we get the desired claim. 
\qed
\end{proof}

\section{Numerical Experiments}\label{sec:expes}


%

We conducted numerical experiments to verify our theoretical finding, by evaluating the convergence of networks with different architecture parameters in the noise-free context. Every network was initialized in accordance with the assumptions of our work and we used the sigmoid activation function. Both $\fop$ and $\xvc$ entries were drawn i.i.d from $\stddistrib$. We used gradient descent to optimize the networks with a fixed step size of 1. A network was trained until it reached a loss of $10^{-7}$ or after 25000 optimization steps. For each set of architecture parameters, we did 50 runs and calculated the frequency at which the network arrived at the error threshold of $10^{-7}$.

\begin{figure}[htb!]
    \centering
    \includegraphics[width=.7\linewidth]{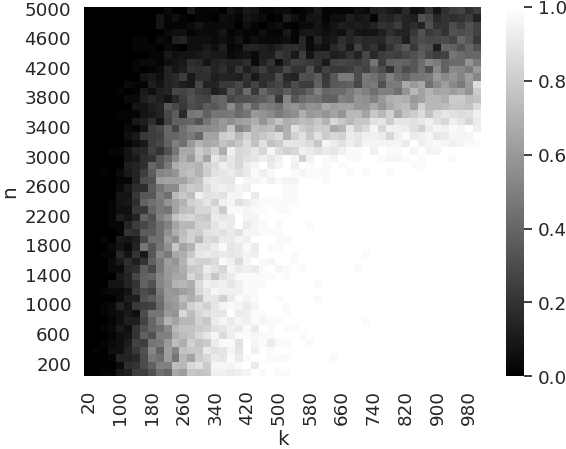}
    \caption{Probability of arriving at a zero loss solution for networks with fixed number of observations $m$, yet varying number of neurons $k$ and signal size~$n$. This emphasizes that the required level of over-parametrization scales at least quadatically with $n$.}
    \label{fig:plot_1}
\end{figure}

We present in Figure~\ref{fig:plot_1} a first experiment where we fix the number of observations $m=10$ and the input size $d=500$, and we let the number of neurons~$k$ and the signal size $n$ vary. It can be observed in this experiment that for any value of $n$, a zero-loss solution is reached with high probability as long as $k$ is ``large enough'', where the phase transition seems to follow a quadratic law. Given that in this setup $n \gg m$ and $\fop$ is Gaussian, this empirical observation is consistent with the theoretical quadratic relation $k \gtrsim n^2m$ which is predicted by our main theorem. However, one may be surprised by the wide range of values of $n$ which can be handled with a fixed $k$. Consider for instance the case $k=900$: convergence is attained for values of $n$ up to $3000$, which includes cases where $k < n$. This goes against our intuition for such underparametrized cases. 

\begin{figure}[htb!]
    \centering
    \includegraphics[width=.7\linewidth]{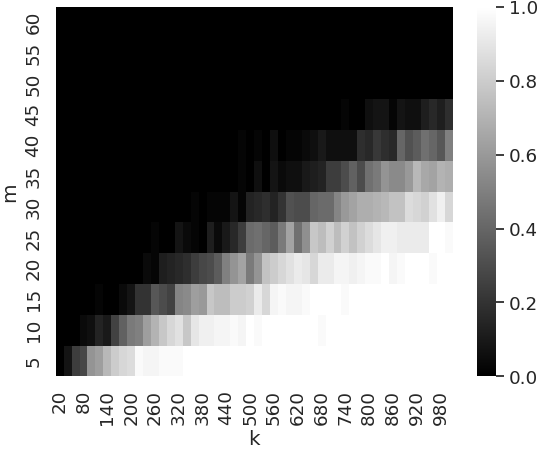}
    \caption{Probability of arriving at a zero loss solution for networks with fixed signal size $n$, and varying number of neurons $k$ and of observations $m$. This emphasizes that the required level of over-parametrization scales linearly with $m$.}
    \label{fig:plot_2}
\end{figure}

Figure~\ref{fig:plot_2} presents a second experiment, where we now fix $n=60$ (still with $d=500$), while letting $k$ vary with $m$. Therein, the expected linear relation between $k$ and $m$ clearly appears, which provides another empirical validation for our theoretical bound. Now, let us consider again a fixed level of over-parametrization, e.g., $k=900$. Contrarily to the previous experiment on the signal size $n$, the range of observations number $m$ which can be tackled is more restricted (here, convergence is observed for values of $m$ up to $25$). For problems where the ratio $m/n$ largely deviates from zero, the level of required over-parametrization is thus much more important than for problems where $n \gg m$. Overall, these experiments validate the order of magnitude of our theoretical bounds, although they also emphasize that these bounds are not really tight.

\section{Conclusion and Future Work}\label{sec:conclu}


This paper studied the convergence of shallow DIP networks and provided bounds on the level of overparametrization, both in the input dimension and the hidden layer dimension, under which the method converges exponentially fast to a zero-loss solution. The proof relies on bounding the minimum singular values of the Jacobian of the network through an overparametrization that ensures a good initialization. These bounds are not tight, as demonstrated by the numerical experiments, but they provide an important step towards the theoretical understanding of DIP methods, and neural networks for inverse problems resolution in general. In the future, this work will be extended in several directions. First, we will study recovery guarantees of the signal $\xvc$. Second, we will investigate the DIP model with unrestricted linear layers and possibly in the multilayer setting.

\medskip

\subsubsection{Acknowledgements} The authors thank the French National Research Agency (ANR) for funding the project ANR-19-CHIA-0017-01-DEEP-VISION.

\newpage

\printbibliography

\end{document}